\makeatletter

\newcommand{\Correct}{\textnormal{\textsf{Correction-type}}\xspace}
\newcommand{\Bregman}{\textnormal{\textsf{Bregman-type}}\xspace}
\def\input@path{{utility/}}
\makeatother

\PassOptionsToPackage{notheorems,nosubfloats}{templateutils}
\documentclass[final]{template} %

\templateproceedings{}{}
\templatepages{}
\title[Gradient-Variation Online Learning]{A Simple and Unified Method for Gradient-Variation Regret}
\usepackage{times}
\usepackage{parskip}

\let\For\relax

\usepackage[utf8]{inputenc} %
\usepackage[T1]{fontenc}    %
\usepackage{lipsum,booktabs}
\usepackage{amsmath,mathrsfs,amssymb,amsfonts,bm}
\usepackage{rotating}
\usepackage{pdflscape}
\usepackage{tcolorbox}
\usepackage{hyperref,url,dsfont,nicefrac}
\usepackage{multirow}
\usepackage{makecell}
\usepackage{anyfontsize}
\usepackage{bbding}
\usepackage{paralist}
\usepackage{enumerate}
\hypersetup{
    colorlinks,
    breaklinks,
    linkcolor = blue,
    citecolor = blue,
    urlcolor  = black,
}
\allowdisplaybreaks
\usepackage{appendix}
\usepackage{multirow,makecell,tabularx}
\usepackage{xfrac}
\usepackage{nicefrac}
\usepackage{threeparttable}
\usepackage{pifont}
\usepackage{tikz}
\usepackage{wrapfig}
\usepackage{enumitem}
\usepackage{circledsteps}

\usepackage{algorithm}
\usepackage{algpseudocode}

\newlength\myindent
\renewcommand{\tilde}{\widetilde}
\renewcommand{\hat}{\widehat}

\def \A {\mathcal{A}}

\def \Bc {\mathcal{B}}

\def \D {\mathcal{D}}

\def \H {\mathcal{H}}

\def \M {\mathcal{M}}

\def \O {\mathcal{O}}

\def \R {\mathbb{R}}

\def \Vb {\bar{V}}

\def \X {\mathcal{X}}
\def \Y {\mathcal{Y}}

\def \a {\mathbf{a}}
\def \b {\mathbf{b}}

\def \c {\mathbf{c}}

\def \g {\mathbf{g}}

\def \m {\boldsymbol{m}}

\def \p {\boldsymbol{p}}

\def \r {\mathbf{r}}

\def \u {\mathbf{u}}
\def \v {\mathbf{v}}

\def \x {\mathbf{x}}
\def \y {\mathbf{y}}

\def \Vt{\tilde{V}}

\def \xh {\hat{\x}}

\def \xs {\x^\star}

\def \epsilon {\varepsilon}

\def \sumT {\sum_{t=1}^T}

\def \sumTT {\sum_{t=2}^T}

\def \sumN {\sum_{i=1}^N}

\def \is {{i^\star}}

\def \ellb {\boldsymbol{\ell}}

\def \Reg {\textsc{Reg}}

\def \DReg {\textsc{D-Reg}}

\def \OOMD {\textsc{OOMD}\xspace}

\def \exp {\text{exp}}
\def \sc {\text{sc}}
\def \les {\lesssim}

\newcommand{\rom}[1]{\setcounter{romancounter}{#1}\textit{(\roman{romancounter})}}

\usepackage{mathtools}
\usepackage{autobreak}
\let\norm\undefined %
\newcommand\norm[1]{\left\| #1 \right\|}

\newcommand\inner[2]{\langle #1, #2 \rangle}
\newcommand\ceil[1]{\lceil #1 \rceil}

\newcommand\sbr[1]{\left( #1 \right)}
\newcommand\mbr[1]{\left[ #1 \right]}
\newcommand\bbr[1]{\left\{ #1 \right\}}

\DeclareMathOperator*{\argmin}{arg\,min}

\def \define {\triangleq}

\def \zero {\mathbf{0}}

\numberwithin{equation}{section}

\usepackage{amsthm}

\newenvironment{proof}{\par\noindent\textbf{Proof.}\ }{\qedbox\par\smallskip}

\newtheorem{Theorem}{Theorem}
\newtheorem{Lemma}{Lemma}
\newtheorem{Corollary}{Corollary}

\theoremstyle{definition}

\newtheorem{Proposition}{Proposition}

\newtheorem{Definition}{Definition}

\newtheorem{Assumption}{Assumption}

\definecolor{darkgreen}{RGB}{0,100,0}
\newtheorem{NoteTemp}{Note}

\newtheorem*{ProofSketch}{Proof Sketch}
\AddToHook{env/ProofSketch/end}{\qedbox}

\usepackage{graphicx,subfigure} %

\definecolor{myred}{RGB}{192,0,1}

\definecolor{myblue}{RGB}{68,114,196}

\definecolor{mygreen}{RGB}{0,128,0}

\definecolor{myorange}{RGB}{230,120,0}

\definecolor{mypurple}{RGB}{110,0,160}

\usepackage{prettyref}
\newcommand{\pref}[1]{\prettyref{#1}}

\newcommand{\savehyperref}[2]{\texorpdfstring{\hyperref[#1]{#2}}{#2}}
\newrefformat{eq}{\savehyperref{#1}{Eq.~\textup{(\ref*{#1})}}}
\newrefformat{eqn}{\savehyperref{#1}{Eq.~(\ref*{#1})}}
\newrefformat{lem}{\savehyperref{#1}{Lemma~\ref*{#1}}}
\newrefformat{lemma}{\savehyperref{#1}{Lemma~\ref*{#1}}}
\newrefformat{def}{\savehyperref{#1}{Definition~\ref*{#1}}}
\newrefformat{line}{\savehyperref{#1}{Line~\ref*{#1}}}
\newrefformat{thm}{\savehyperref{#1}{Theorem~\ref*{#1}}}
\newrefformat{cond}{\savehyperref{#1}{Condition~\ref*{#1}}}
\newrefformat{table}{\savehyperref{#1}{Table~\ref*{#1}}}
\newrefformat{corr}{\savehyperref{#1}{Corollary~\ref*{#1}}}
\newrefformat{cor}{\savehyperref{#1}{Corollary~\ref*{#1}}}
\newrefformat{sec}{\savehyperref{#1}{Section~\ref*{#1}}}
\newrefformat{subsec}{\savehyperref{#1}{Section~\ref*{#1}}}
\newrefformat{app}{\savehyperref{#1}{Appendix~\ref*{#1}}}
\newrefformat{appendix}{\savehyperref{#1}{Appendix~\ref*{#1}}}
\newrefformat{ass}{\savehyperref{#1}{Assumption~\ref*{#1}}}
\newrefformat{assum}{\savehyperref{#1}{Assumption~\ref*{#1}}}
\newrefformat{ex}{\savehyperref{#1}{Example~\ref*{#1}}}
\newrefformat{fig}{\savehyperref{#1}{Figure~\ref*{#1}}}
\newrefformat{alg}{\savehyperref{#1}{Algorithm~\ref*{#1}}}
\newrefformat{remark}{\savehyperref{#1}{Remark~\ref*{#1}}}
\newrefformat{conj}{\savehyperref{#1}{Conjecture~\ref*{#1}}}
\newrefformat{prop}{\savehyperref{#1}{Proposition~\ref*{#1}}}
\newrefformat{proto}{\savehyperref{#1}{Protocol~\ref*{#1}}}
\newrefformat{prob}{\savehyperref{#1}{Problem~\ref*{#1}}}
\newrefformat{claim}{\savehyperref{#1}{Claim~\ref*{#1}}}
\newrefformat{que}{\savehyperref{#1}{Question~\ref*{#1}}}
\newrefformat{op}{\savehyperref{#1}{Open Problem~\ref*{#1}}}
\newrefformat{fn}{\savehyperref{#1}{Footnote~\ref*{#1}}}
\newrefformat{require}{\savehyperref{#1}{Requirement~\ref*{#1}}}
\newcommand{\qedbox}{\qed}

\newcounter{romancounter}
\usepackage{colortbl}
\newcommand{\omlprod}{\textnormal{\textsf{Optimistic-Adapt-ML-Prod}}\xspace}

\newrefformat{tab}{\savehyperref{#1}{Table~\ref*{#1}}}
\renewcommand{\Reg}{\textnormal{\textsc{Reg}}}
\renewcommand{\DReg}{\textnormal{\textsc{D-Reg}}}

\def \OOMD {\textnormal{\textsf{OOMD}}\xspace}
\def \OONS {\textnormal{\textsf{OONS}}\xspace}
\newcommand{\MetaReg}{\textnormal{\textsc{Meta-Reg}}}
\newcommand{\BaseReg}{\textnormal{\textsc{Base-Reg}}}
\def \OOGD {\textnormal{\textsf{OOGD}}\xspace}

\def \Swordpp {\textnormal{\textsf{Sword++}}\xspace}
\def \Swordo {\textnormal{\textsf{Sword.optimism}}\xspace}
\def \scvx {\textnormal{sc}}
\def \exp {\textnormal{exp}}
\def \cvx {\textnormal{c}}

\templateSuppressPackageChecks
\hypersetup{hypertexnames=false}

\templateauthor{%
 \Name{Hang Yu} \Email{yuhang@lamda.nju.edu.cn}\\
 \Name{Peng Zhao} \Email{zhaop@lamda.nju.edu.cn}\\
 \addr State Key Laboratory for Novel Software Technology, Nanjing University, China \\
 School of Artificial Intelligence, Nanjing University, China
}

\begin{document}

\maketitle

\begin{abstract}%
Gradient-variation online learning has attracted increasing attention for its connections to game theory and optimization.
It has been studied in both non-stationary online learning, which aims to minimize dynamic regret against time-varying comparator sequences, and universal online learning, which seeks to adapt to unknown loss curvature.
Prior work obtained gradient-variation bounds using different techniques tailored to each setting.
We propose a  \emph{simple} and \emph{unified} method based on surrogate losses with a quadratic penalty on decision movement, yielding gradient-variation bounds in both settings.
For non-stationary online learning, the method induces an optimistic-to-gradient-variation reduction, enabling a broad class of algorithms to match the rate established by \citet{zhao2024adaptivity}.
For universal online learning, it yields a \emph{simple} algorithm that matches the \emph{optimal} universal regret bounds of \citet{yan2024simple} while relaxing the extended smoothness assumption.
Finally, we apply our method to online games and improve on existing guarantees.
\end{abstract}

\section{Introduction}
Online convex optimization (OCO) provides a versatile framework for studying sequential interactions between a learner and an environment \citep{hazan2016introduction,orabona2019modern}. At each round $t \in [T]$, the learner selects a decision $\x_t$ from a compact convex set $\X \subseteq \R^d$, while the environment simultaneously chooses a convex loss function $f_t: \X \rightarrow \R$.
The learner then incurs loss $f_t(\x_t)$ and observes the gradient $\nabla f_t(\x_t)$, aiming to optimize the \mbox{game-theoretical} performance measure known as static regret~\citep{cesa2006prediction}, which is defined as
\begin{equation}
  \label{eq:regret}
  \Reg_T = \sumT f_t(\x_t) - \min_{\x \in \X} \sumT f_t(\x).
\end{equation}
For OCO, the minimax optimal regret bounds are $\O(\sqrt{T})$ for convex functions~\citep{zinkevich2003online}, $\O(\frac{1}{\lambda}\log T)$ for $\lambda$-strongly convex functions \citep{hazan2016introduction}, and $\O(\frac{d}{\alpha}\log T)$ for $\alpha$-exp-concave functions~\citep{yang2014regret}.
Beyond worst-case minimax optimality, prior work studied regret bounds that adapt to problem-dependent hardness.
Among these problem-dependent quantities, \emph{gradient variation}~\citep{chiang2012online, yang2014regret} is of particular interest.
It quantifies the cumulative variation in gradients between consecutive loss functions:
\begin{equation}
    \label{eq:VT}
    V_T \define \sumTT \sup_{\x \in \X} \|\nabla f_t(\x) - \nabla f_{t-1}(\x)\|^2_2.
\end{equation}
Adapting to gradient variation improves regret bounds to $\O(\sqrt{V_T})$, $\O(\frac{1}{\lambda}\log V_T)$, and $\O(\frac{d}{\alpha}\log V_T)$ for convex, $\lambda$-strongly convex, and $\alpha$-exp-concave loss functions, respectively.
Gradient-variation adaptivity also helps bridge adversarial and stochastic optimization~\citep{Sarah2022between,chen2024optimistic}, enables fast rates in games~\citep{rakhlin2013optimization,syrgkanis2015fast}, and facilitates acceleration in smooth offline optimization~\citep{cutkosky2019anytime,kavis2019unixgrad,zhao2025gradient}, among other applications. Recent work further explored gradient variation in more challenging settings, including non-stationary online learning and universal online learning.

\textbf{Non-stationary Online Learning.} Static regret~\eqref{eq:regret} measures performance against a fixed comparator $\xs\in\argmin_{\x\in\X}\sumT f_t(\x)$ and becomes an inadequate measure in non-stationary environments where the optimal decision may change over time. To address this limitation, prior work studied dynamic regret~\citep{zinkevich2003online}, which compares the learner's performance with a sequence of \mbox{time-varying} comparators \smash{$\{\u_t\}_{t=1}^T$} and is defined as follows:
\begin{equation}
    \label{eq:dreg}
    \DReg_T(\u_{1:T}) \define\sumT f_t(\x_t) - \sumT f_t(\u_t).
\end{equation}
A desirable dynamic regret bound adapts to the comparator path length $P_T\define\sum_{t=2}^T\norm{\u_t-\u_{t-1}}_2$, which measures cumulative changes in the comparator sequence and is unknown to the learner.
 \citet{zhang2018adaptive} obtained the minimax-optimal dynamic regret bound \smash{$\O(\sqrt{T(1+P_T)})$}, which was later strengthened for smooth functions to  $\O(\sqrt{(1+V_T+P_T)(1+P_T)})$~\citep{zhao2024adaptivity}.

\textbf{Universal Online Learning.} Classical approaches require prior knowledge of the curvature of loss functions to select suitable algorithms and obtain the corresponding regret guarantees. Such information, however, may be unavailable in practice. Universal online learning aims to develop a single algorithm that adapts to unknown curvature and matches the regret guarantees achievable when the curvature is known~\citep{van2016metagrad,wang2019adaptivity,zhang2022simple,yan2023universal,yan2024simple}. 
\citet{van2016metagrad} initiated the study of universal online learning by developing a single algorithm with regret bounds of $\O(\frac{d}{\alpha}\log T)$ for $\alpha$-exp-concave functions and $\O(\sqrt{T})$ for convex functions. 
\citet{wang2019adaptivity} subsequently obtained $\O(\frac{1}{\lambda}\log T)$ for $\lambda$-strongly convex functions while retaining optimal regret guarantees for convex and exp-concave functions. 
Exploiting smoothness, \citet{zhang2022simple} refined minimax guarantees to optimal gradient-variation bounds for strongly convex and exp-concave functions. 
\citet{yan2023universal} further obtained a regret of $\O(\sqrt{V_T\log V_T})$ bound for convex functions. Under an extended smoothness assumption, \citet{yan2024simple} achieved optimal universal gradient-variation bounds for all three classes of functions.

\textbf{Our Contributions.} In this work, we propose a \emph{simple} and \emph{unified} method for gradient-variation regret bounds in these challenging settings. For non-stationary online learning, our method yields a \emph{black-box reduction} that converts optimistic regret bounds into gradient-variation bounds. This reduction enables a broad class of algorithms to achieve $\O(\sqrt{(1+V_T+P_T)(1+P_T)})$, matching the rate established by \citet{zhao2024adaptivity}.
For universal online learning, our method simplifies algorithm design and yields an algorithm that matches the \emph{optimal} gradient-variation bounds of \citet{yan2024simple} while relaxing their extended smoothness assumption. 
These bounds are $\O(\sqrt{V_T})$, $\O(\frac{d}{\alpha}\log V_T)$, and $\O(\frac{1}{\lambda}\log V_T)$ for convex, $\alpha$-exp-concave, and $\lambda$-strongly convex functions, respectively. \pref{table:universal-comparison} compares these guarantees with prior work. 
Notably, our algorithm preserves the ``Regret bounded by Variation in Utilities'' (RVU) property, which is crucial for fast convergence in games, while requiring minimal smoothness assumptions.
We further apply our universal regret guarantee to facilitate fast convergence in online games.
Under honest play, the sum of the two players' regrets is $\O(1)$ in both bilinear and strongly convex--strongly concave games, matching \citet{yan2023universal}. Under dishonest play, the individual regret guarantee in bilinear games improves from $\O(\sqrt{T\log T})$ in \citet{yan2023universal} to $\O(\sqrt{T})$.

\textbf{Organization.} The remainder of the paper is organized as follows. \pref{sec:preliminary} introduces the preliminaries and reviews existing approaches to gradient-variation bounds. \pref{sec:approach} presents our unified method for deriving gradient-variation regret bounds. \pref{sec:nonstationary} establishes a black-box reduction for non-stationary online learning. \pref{sec:universality} develops a simple universal algorithm with optimal gradient-variation regret bounds. \pref{sec:applications-games} applies these universal guarantees to online games. Finally, \pref{sec:Summary} concludes the paper. All proofs are deferred to the appendices.

\begin{table}[!t]
\centering
\caption{\small{Comparison of universal regret guarantees. The regret columns show bounds for strongly convex, exp-concave, and convex loss functions using $\O(\cdot)$ notation. We list universal guarantees in terms of the gradient variation $V_T$ or the time horizon $T$. The ``Regret bounded by Variation in Utilities'' (RVU) column indicates whether the method supports this property for fast convergence in online games. In the Smoothness column, ``No'' means no smoothness assumption, ``Yes'' means smoothness on $\X$, and ``Extended'' means smoothness on a larger domain containing $\X$.}}
\label{table:universal-comparison}
\vspace{1mm}
\begingroup
\small
\renewcommand{\arraystretch}{1.6}
\setlength{\tabcolsep}{4pt}
\resizebox{0.95\linewidth}{!}{%
\begin{tabular}{c|ccc|c|c}
\hline
\hline
\multirow{2}{*}{\textbf{Works}} & \multicolumn{3}{c|}{\textbf{Regret Bounds}} & \multirow{2}{*}{\textbf{RVU}} & \multirow{2}{*}{\textbf{Smoothness}}\\
\cline{2-4}
 & \makecell{Strongly Convex} & Exp-concave & Convex & & \\
\hline
\citet{zhang2022simple} & $\O(\log V_T)$ & $\O(d\log V_T)$ & $\O(\sqrt{T})$ & No & Yes\\
\hline
\citet{yan2023universal} & $\O(\log V_T)$ & $\O(d\log V_T)$ & $\O(\sqrt{V_T\log V_T})$ & Yes & Yes\\
\hline
\citet{yan2024simple} & $\O(\log V_T)$ & $\O(d\log V_T)$ & $\O(\sqrt{V_T})$ & No & Extended\\
\hline
\hline
\rowcolor{gray!13}\textbf{Ours} & $\O(\log V_T)$ & $\O(d\log V_T)$ & $\O(\sqrt{V_T})$ & Yes & Yes\\
\hline
\hline
\end{tabular}}
\endgroup
\end{table}

\section{Preliminaries}
\label{sec:preliminary}
This section introduces the notation, assumptions, and definitions used throughout the paper and reviews two existing approaches to deriving gradient-variation regret bounds.
\vspace{-3mm}
\subsection{Notation, Assumptions, and Definitions}
We use $\|\cdot\|$ for $\|\cdot\|_2$ by default.
For $\v\in\R^d$, let $v_i$ denote its $i$-th coordinate.
For a strictly convex and differentiable function $\psi:\X\rightarrow \R$, the induced Bregman divergence is defined as $\D_\psi(\x, \y) \define \psi(\x) - \psi(\y) - \inner{\nabla \psi(\y)}{\x - \y}$.
We write $a\les b$ or $a=\O(b)$ when there exists a constant $C<\infty$ such that $a\le Cb$.
We use $\O(\cdot)$ notation to highlight the dependence on $T$ and $V_T$ while treating iterated logarithmic factors as constants~\citep{luo2015achieving}.

\begin{Assumption}[Boundedness]
    \label{assum:boundedness}
    The domain $\X\subseteq\R^d$ has diameter at most $D$, so $\|\x-\y\|\le D$ for all $\x,\y\in\X$. For every $t\in[T]$ and $\x\in\X$, the loss gradient satisfies $\|\nabla f_t(\x)\|\le G$.
\end{Assumption}
\begin{Assumption}[Smoothness]
    \label{assum:smoothness}
    Each loss function $f_t$ is $L$-smooth on $\X$, meaning that $\|\nabla f_t(\x)-\nabla f_t(\y)\|\le L\|\x-\y\|$ for all $\x,\y\in\X$.
\end{Assumption}
The domain boundedness and gradient boundedness are standard assumptions for regret minimization in OCO~\citep{shalev2007online,hazan2016introduction}.
The smoothness assumption  is necessary for first-order algorithms to achieve the gradient-variation regret~\citep{chiang2012online}.

In the following, we provide formal definitions of strong convexity and exp-concavity.
\begin{Definition}[Strong Convexity]
    \label{def:sconvex}
    A function $f(\cdot)$ is $\lambda$-strongly convex if $f(\x) - f(\y) \le \inner{\nabla f(\x)}{\x - \y} - \frac{\lambda}{2} \cdot \|\x - \y\|^2$ holds for any $\x,\y \in \X$.
\end{Definition}
\begin{Definition}[Exp-concavity]
    \label{def:expconcave}
    A function $f(\cdot)$ is $\alpha$-exponentially concave (abbreviated as \emph{exp-concave}) if $f(\x) - f(\y) \le \inner{\nabla f(\x)}{\x - \y} - \frac{\alpha}{2} \cdot \inner{\nabla f(\x)}{\x - \y}^2$ holds for any $\x,\y \in \X$.
\end{Definition}
The formal definition of $\beta$-exp-concavity is that $\exp(-\beta f(\cdot))$ is concave. Under Assumption~\ref{assum:boundedness}, $\beta$-exp-concavity implies \pref{def:expconcave} with $\alpha = \frac{1}{2} \cdot \min\{1/(4GD), \beta\}$~{\citep[Lemma~4.3]{hazan2016introduction}}. For clarity and simplicity, we adopt \pref{def:expconcave} as an alternative definition of exp-concavity.

\subsection{Review of Existing Approaches}

A standard approach to deriving gradient-variation regret bounds first controls regret in terms of the empirical gradient variation
\begin{align*}
  \Vb_T\define\sumTT\norm{\nabla f_t(\x_t)-\nabla f_{t-1}(\x_{t-1})}^2.
\end{align*}
To express the resulting bound in terms of $V_T$, existing analyses decompose $\Vb_T$ into a gradient-variation term and a term involving the learner's decision trajectory.
We summarize two types of analysis adopted in prior work for this decomposition: \Correct and \Bregman~\citep{chiang2012online,rakhlin2013optimization,syrgkanis2015fast,zhao2020dynamic,yan2023universal,yan2024simple,xie2024gradient,zhao2024adaptivity,zhao2025adaptivity,zhao2026gradient}.

\paragraph{Correction-type Analysis} Inserting $\nabla f_{t-1}(\x_t)$ gives
\begin{align}
  \Vb_T\le {}&2\sumTT\norm{\nabla f_t(\x_t)-\nabla f_{t-1}(\x_t)}^2
  +2\sumTT\norm{\nabla f_{t-1}(\x_t)-\nabla f_{t-1}(\x_{t-1})}^2\nonumber\\
  \le {}&2V_T+2L^2\sumTT\norm{\x_t-\x_{t-1}}^2,
  \label{eq:method-movement-decomposition}
\end{align}
where the first inequality follows from the AM--GM inequality, and the second uses $L$-smoothness and the definition of $V_T$~\eqref{eq:VT}.
To obtain a regret bound in terms of $V_T$ alone, the key is to cancel the positive movement term $\sumTT\norm{\x_t-\x_{t-1}}^2$ in \pref{eq:method-movement-decomposition} with a matching negative term.

For static regret, optimistic online mirror descent (\OOMD) provides the required negative movement terms when the loss function class and curvature parameters are known~\citep{chiang2012online}.
With a suitable configuration for each class, it achieves optimal gradient-variation regret bounds for convex, strongly convex, and exp-concave loss functions.
These negative movement terms are also a key component of the RVU property, which enables fast convergence in online games~\citep{rakhlin2013online,syrgkanis2015fast}.
The cancellation required by \Correct analysis is more difficult in universal online learning.
\citet{yan2023universal} made progress in this direction using a three-layer ensemble with cascaded corrections, but their universal regret bounds remain suboptimal.
For dynamic regret, \citet{zhao2024adaptivity} designed surrogate losses for the base learners in a meta-base ensemble to cancel the movement term and establish a gradient-variation dynamic regret bound.
Across the different settings discussed above, these works shared the \Correct decomposition but used different mechanisms to obtain the negative terms required for cancellation.

\paragraph{Bregman-type Analysis} For a fixed comparator $\u\in\X$, \Bregman decomposes $\Vb_T$ by inserting the gradients $\nabla f_t(\u)$ and $\nabla f_{t-1}(\u)$:
\begin{align}
  \Vb_T\le {}&3\sumTT\Bigl(\norm{\nabla f_t(\x_t)-\nabla f_t(\u)}^2
  +\norm{\nabla f_t(\u)-\nabla f_{t-1}(\u)}^2+\norm{\nabla f_{t-1}(\u)-\nabla f_{t-1}(\x_{t-1})}^2\Bigr)\nonumber\\
  \le {}&6L\sumTT\D_{f_t}(\u,\x_t)+3V_T
  +6L\sumTT\D_{f_{t-1}}(\u,\x_{t-1})
  \le3V_T+12L\sumT\D_{f_t}(\u,\x_t),
  \label{eq:method-bregman-decomposition}
\end{align}
where the first inequality follows from the AM--GM inequality and the second from Theorem 2.1.5 of \citet{Nesterov2018convexoptimization}.
Convexity yields matching negative Bregman-divergence terms in the regret decomposition that can cancel the positive terms in \pref{eq:method-bregman-decomposition}.
Using this analysis, \citet{yan2024simple} achieved optimal universal regret bounds for all three function classes.
This analysis has two limitations. First, applying Theorem 2.1.5 of \citet{Nesterov2018convexoptimization} requires $L$-smoothness on an extended domain $\X^+$ containing $\X$, a stronger condition than \pref{assum:smoothness}. Second, the negative Bregman-divergence terms do not directly support fast convergence in online games~\citep{yan2024simple}.

\section{A General Method for Negative Movement}
\label{sec:approach}
Our goal is to control the positive movement term in \pref{eq:method-movement-decomposition} while minimizing regret.
This motivates a Lagrangian viewpoint: we treat cumulative movement $\sumTT\norm{\x_t-\x_{t-1}}^2$ as an auxiliary budget constraint and incorporate a quadratic penalty into the original loss function, yielding the surrogate loss functions
\begin{align}
  \bar f_t(\x)\define f_t(\x)+\frac{\gamma}{2}\norm{\x-\x_{t-1}}^2,
  \label{eq:method-surrogate}
\end{align}
where $\gamma>0$ is a parameter to be specified and $\x_0=\x_1$.
The quadratic penalty discourages movement from the preceding decision and induces a negative movement term in the analysis. 
Taking convex functions as an example, we have the following property:
\begin{Proposition}
\label{prop:approach}
Under \pref{assum:boundedness}, let $f_t$ be convex for every $t\in[T]$. Let $\v_t\define\nabla f_t(\x_t)+\gamma(\x_t-\x_{t-1})$ be the \emph{corrected} gradient.
For any $\gamma>0$ and every decision sequence $\x_{1:T}\in\X^T$ and comparator sequence $\u_{1:T}\in\X^T$, we always have
\begin{align}
   \sumT \inner{\nabla f_t(\x_t)}{\x_t-\u_t}=
    \underbrace{\sumT\inner{\v_t}{\x_t-\u_t}}_{\text{Linearized regret}}
    -\underbrace{\frac\gamma2\sumTT\norm{\x_t-\x_{t-1}}^2}_{\text{Movement}}
    +\O(P_T),
    \label{eq:dynamic-app}
\end{align}
where $P_T\define\sum_{t=2}^T\norm{\u_t-\u_{t-1}}$ is the comparator path length.
\end{Proposition}
\begin{ProofSketch} By the definition of $\v_t$, we have
\begin{align*}
  &\sum_{t=1}^T\inner{\nabla f_t(\x_t)}{\x_t-\u_t}\notag
  =\sum_{t=1}^T\inner{\v_t}{\x_t-\u_t}
  -\gamma\sum_{t=1}^T\inner{\x_t-\x_{t-1}}{\x_t-\u_t}\notag\\
  ={}&\sum_{t=1}^T\inner{\v_t}{\x_t-\u_t}
  -\frac\gamma2\sumTT\norm{\x_t-\x_{t-1}}^2
  -\frac\gamma2\sum_{t=1}^T
  \left(\norm{\x_t-\u_t}^2-\norm{\x_{t-1}-\u_t}^2\right),
\end{align*}
where the second equality follows from the three-point identity $2\inner{\a-\b}{\a-\c}
  =\norm{\a-\c}^2-\norm{\b-\c}^2+\norm{\a-\b}^2$, for $\a,\b,\c\in\R^d$ and the last term of \pref{eq:dynamic-app} is bounded by $\O(P_T)$.
\end{ProofSketch}
Note that we can generalize \pref{prop:approach} to general Bregman divergence in \pref{app:bregman-movement}.

The negative movement in \pref{eq:dynamic-app} plays an important role in optimistic online learning. In addition to a gradient sequence $\{\g_t\}_{t=1}^T$, an optimistic online learning algorithm has access to a prediction $\m_{t+1}$ of the next gradient $\g_{t+1}$, called an \emph{optimism}. By \pref{prop:approach}, any algorithm with a standard optimistic regret bound scaling with $\sqrt{\sum_t\|\g_t-\m_t\|^2}$ satisfies the ``Regret bounded by Variation in Utilities'' (RVU) property when applied to the surrogate losses in \pref{eq:method-surrogate} with optimism $\m_t=\g_{t-1}$. This property is defined as follows:
\begin{Definition}[RVU Property]
  \label{def:RVU}
    An algorithm with a parameter $\eta>0$ satisfies the RVU property if there exist constants $\alpha,\beta,\omega>0$ such that the regret $\Reg_T$ on decision sequence $\{\x_t\}_{t=1}^T$ and gradient sequence $\{\g_t\}_{t=1}^T$ is bounded by
    \begin{align}
        \Reg_T\le \frac{\alpha}{\eta}+ \beta\eta\sumTT \|\g_t-\g_{t-1}\|_\infty^2-\frac{\omega}{\eta} \sumTT\|\x_t-\x_{t-1}\|_1^2.
    \label{eq:RVU}
    \end{align}
\end{Definition}

\textbf{A Toy Example.} We illustrate our method with \omlprod~\citep{wei2016tracking}. Its original optimistic regret bound lacks the negative movement term required by the RVU property and therefore does not directly establish this property. Our gradient correction introduces the required term through \pref{prop:approach}. Specifically, with the previous gradient as the optimism, Theorem 3.4 of \citet{wei2016tracking} gives the regret bound $\Reg_T\le\O\big(\sqrt{\sum_t\norm{\g_t-\g_{t-1}}_\infty^2}\big)$.
At each round $t$, we update the algorithm using the corrected gradient $\v_t=\g_t+\frac{1}{32\eta}(\x_t-\x_{t-1})$ and set the optimism to $\m_t=\v_{t-1}$. Combining the original optimistic regret bound with \pref{prop:approach} yields $\Reg_T \le \O\big(\sqrt{\sum_t \norm{\g_t-\g_{t-1}}_\infty^2}-\frac{1}{\eta}\sum_t \norm{\x_t-\x_{t-1}}_1^2\big)$, establishing the RVU property for the modified algorithm.
Thus, our method extends the applicability of the RVU property to optimistic algorithms for which a negative movement term was not previously established.

In the next two sections, we further illustrate how the negative movement term in \pref{prop:approach} plays a key role in two settings: non-stationary online learning and universal online learning.

\section{Gradient-Variation Non-Stationary Online Learning}
\label{sec:nonstationary}
In the non-stationary setting, \pref{prop:approach} reduces the task of obtaining a gradient-variation dynamic regret bound to designing an algorithm with an optimistic dynamic regret bound.

\begin{Theorem}[Optimistic-to-Gradient-Variation Reduction]
\label{thm:optimistic-gradient-variation-reduction}
Under \pref{assum:boundedness} and \ref{assum:smoothness}, let $\u_{1:T}$ be a comparator sequence in $\X$ and let $P_T\define\sumTT\norm{\u_t-\u_{t-1}}$ be the path length. Suppose that the functions $\{f_t\}_{t=1}^T$ are convex and that an online algorithm $\A$ satisfies the following bound for any gradient sequence $\{\g_t\}_{t=1}^T$ and optimism sequence $\{\m_t\}_{t=1}^T$:
\begin{align}
   \sumT  \inner{\g_t}{\x_t-\u_t}
    \lesssim \sqrt{(1+P_T)\sbr{1+P_T+\sumT\norm{\g_t-\m_t}^2}}.
    \label{eq:dynamic-optimistic}
\end{align}
Then, when run with $\g_t=\nabla f_t(\x_t)+L(\x_t-\x_{t-1})$ and either of the following optimism,
\begin{align}
    \text{\rom{1} } \m_t=\nabla f_{t-1}(\x_{t-1})+L(\x_{t-1}-\x_{t-2}),\qquad
    \text{\rom{2} } \m_t=\nabla f_{t-1}(\x_{t-1}), \label{eq:mtchoice}
\end{align}
$\A$ satisfies the gradient-variation dynamic regret bound:
\begin{equation*}
    \DReg_T(\u_{1:T}) \le \O\sbr{\sqrt{(1+P_T)(1+P_T+V_T)}}.
\end{equation*}
\end{Theorem}
The proof is given in \pref{app:optimistic-gradient-variation-reduction}, where we apply the optimistic dynamic regret bound in \pref{eq:dynamic-optimistic} to the linearized regret in \pref{eq:dynamic-app} and then use the negative movement term in \pref{eq:dynamic-app} to cancel the resulting positive terms, yielding the desired bound. \pref{thm:optimistic-gradient-variation-reduction} offers two advantages over the constrained version of the optimistic-to-gradient-variation reduction in \citet[Theorem~1]{zhao2026gradient}. First, our reduction requires smoothness only on $\X$, whereas theirs requires smoothness on an extended domain $\X^+$ containing $\X$ or an unbounded domain. Second, our reduction preserves the RVU property, whereas theirs does not.

Conveniently, the \Swordo algorithm in \citet{zhao2024adaptivity} already satisfies the condition in \pref{thm:optimistic-gradient-variation-reduction}. \Swordo uses an online ensemble framework with a two-layer meta-base architecture to adapt to an unknown comparator sequence. A diverse set of base learners is deployed to explore the environment, while a meta-algorithm adaptively tracks the best-performing base learner. Theorem 9 of \citet{zhao2024adaptivity} guarantees that \Swordo achieves the optimistic dynamic regret bound in \pref{eq:dynamic-optimistic}. We present \Swordo with the corrected gradient and optimism in \pref{alg:dynamic-simple} and provide the proof in \pref{app:optimistic-gradient-variation-reduction}.
 \begin{algorithm}[t]
    \caption{Dynamic Regret with Corrected Gradient}
    \label{alg:dynamic-simple}
    \begin{algorithmic}[1]
        \Require $G,D,L,T>0$.
        \State \textbf{Initialize:} Instantiate $\A_s$ as \Swordo without correction terms.\footnotemark Set $\x_1\in\X$.
        \For{$t=1$ \textbf{to} $T$}
            \State Play $\x_t$, incur loss $f_t(\x_t)$, and observe $\nabla f_t(\x_t)$.
            \State Construct $\g_t=\nabla f_t(\x_t)+L(\x_t-\x_{t-1})$ and optimism $\m_{t+1}$ in \pref{eq:mtchoice}.
            \State Provide $\g_t$ and $\m_{t+1}$ to the meta-learner and all base learners of $\A_s$.
            \State Obtain the next decision $\x_{t+1}$ from $\A_s$.
        \EndFor
    \end{algorithmic}
\end{algorithm}
\footnotetext{\citet{zhao2024adaptivity} introduced correction terms to obtain negative terms in the meta algorithm's analysis. We directly use the negative terms in \pref{prop:approach}, avoiding the need for such corrections.}

\begin{Corollary}
\label{cor:dynamic-gradient-variation}
Under \pref{assum:boundedness} and \ref{assum:smoothness}, suppose that $f_t$ is convex for every $t\in[T]$. Then, for every comparator sequence $\u_{1:T}\in\X^T$, \pref{alg:dynamic-simple} guarantees
\begin{align*}
    \DReg_T(\u_{1:T})
    \le\O\sbr{\sqrt{(1+P_T+V_T)(1+P_T)}}.
\end{align*}
\end{Corollary}
\pref{cor:dynamic-gradient-variation} matches the gradient-variation dynamic regret established by \Swordpp in \citet{zhao2024adaptivity} with a simpler analysis. \Swordpp obtained the negative movement by combining negative terms from meta and base regret bounds, whereas we obtain it directly using the corrected gradient.

\section{Gradient-Variation Universal Online Learning}
\label{sec:universality}

For static regret, obtaining negative movement terms remains the main technical challenge in achieving gradient-variation bounds.
When the function class is known, negative movement terms are relatively easy to obtain. Prior analyses~\citep{chiang2012online,yang2014regret} can therefore establish the optimal gradient-variation bounds directly, without relying on \pref{prop:approach}.

In the universal setting, obtaining the required negative movement term is more challenging. To adapt to unknown function classes and curvature, universal online learning algorithms typically use a two-layer architecture: $N$ base learners are maintained for strongly convex, exp-concave, and convex losses, while a meta learner tracks the best-performing base learner~\citep{zhang2022simple,yan2023universal,yan2024simple}. Specifically, for $\lambda$-strongly convex and $\alpha$-exp-concave losses, these methods constructed candidate pools $\H^{\scvx}$ and $\H^{\exp}$ for the unknown parameters $\lambda$ and $\alpha$, respectively, using a geometric grid:
\begin{equation*}
    \H^{\exp} = \H^{\scvx} \triangleq \left\{ \frac{1}{T}, \frac{2}{T}, \frac{2^2}{T}, \cdots, \frac{2^{n-1}}{T} \right\},
\end{equation*}
where $n = \ceil{\log_2 T} + 1 = \O(\log T)$ is the number of candidates in each pool. One base learner is instantiated for each candidate $\lambda_i$ or $\alpha_i$, with the final base learner, indexed by $N$, assigned to convex losses. Let $\x_{t,i}$ denote the decision of base learner $i$ at round $t$, and let $p_{t,i}$ denote its weight assigned by the meta-algorithm. The meta decision is $\x_t=\sumN p_{t,i}\x_{t,i}$.

 For the meta algorithm, \citet{zhao2025adaptivity} observed that a \emph{second-order optimistic regret} bound with \emph{negative stability terms} is crucial for achieving gradient-variation universality. Specifically, for each base learner $i$, the meta regret should satisfy
\begin{equation}
    \label{eq:second-order-regret}
    \sumT r_{t,i} \le \O\left(\sqrt{\sumT (r_{t,i}-m_{t,i})^2}-\sumTT\norm{\p_t-\p_{t-1}}_1^2\right),
\end{equation}
where $r_{t,i}\define\inner{\g_t}{\x_t-\x_{t,i}}$ is the instantaneous meta regret relative to the $i$-th base learner, and $m_{t,i}$ is the corresponding optimism. For strongly convex and exp-concave base learners, setting $m_{t,i}=0$ allows the negative curvature terms to offset the meta regret. For example, for $\alpha$-exp-concave functions, the negative curvature term $-\frac{\alpha}{4}\sum_t r_{t,i}^2$ yields \smash{$\sqrt{\sum_t r_{t,i}^2}-\frac{\alpha}{4}\sum_t r_{t,i}^2 \le \O(1)$}.
For convex functions, no negative curvature terms are available, so the negative stability term $-\sumTT\norm{\p_t-\p_{t-1}}_1^2$ plays a key role in controlling the meta regret. To obtain the required guarantee, \citet{yan2023universal} designed a complex two-layer meta algorithm based on \textsc{MsMwC-Master}~\citep{chen2021impossible}. However, their regret bound for convex losses still contained an extra $\sqrt{\log V_T}$ factor relative to the optimal rate.

In \pref{prop:approach}, the negative movement term in \pref{eq:dynamic-app} can directly cancel the meta regret, without requiring additional negative terms in the meta regret bound. This simplifies the design of the meta algorithm: any algorithm with a \emph{second-order optimistic regret} bound can serve as the meta algorithm. In what follows, we use \omlprod~\citep{wei2016tracking} as the meta algorithm, for which deriving negative stability terms in the regret bound was previously difficult.

 \textbf{Meta Learner. }\omlprod updates the weight $\p_{t+1}$ as follows:
\begin{equation}
    \label{eq:omlprod}
    \begin{gathered}
       W_{t,i} = \sbr{W_{t-1,i} \cdot \exp\sbr{\epsilon_{t-1,i} r_{t,i} - \epsilon_{t-1,i}^2 (r_{t,i} - m_{t,i})^2}}^{\frac{\epsilon_{t,i}}{\epsilon_{t-1,i}}},\\
        p_{t+1,i} \propto \epsilon_{t,i} \cdot \exp(\epsilon_{t,i} m_{t+1,i}) \cdot W_{t,i},
    \end{gathered}
\end{equation}
where $r_{t,i}=\inner{\g_t}{\x_t-\x_{t,i}}$ is the instantaneous meta regret relative to the $i$-th base learner. The optimistic vector $\m_t \in \R^{N}$ is designed as
\begin{equation}
    \label{eq:Bregman-optimism}
    m_{t,i} =\inner{\g_{t-1}}{\x_{t-1} - \x_{t,i}} \mbox{ for } i =N, \mbox{ and } m_{t,i} = 0 \mbox{ for } i\neq N.
\end{equation}
Assuming that $\|\g_t\|\le H$ for all $t\in[T]$, we set the learning rates as
\begin{align}
\epsilon_{t,i}=\min\left\{\frac{1}{8HD},\,
\sqrt{\frac{\ln N}{4H^2D^2+\sum_{s=1}^t(r_{s,i}-m_{s,i})^2}}\right\},
\qquad \epsilon_{0,i}=\frac{1}{8HD}.
\label{eq:meta-learning-rate}
\end{align}

Following \citet{yan2024simple}, we construct the base learners as follows:

\textbf{Base Learner. }The base learners update their decisions $\{\x_{t,i}\}_{i=1}^{N}$ using \OOMD with their respective surrogate losses and optimisms. 
We construct the surrogate loss for each function class as follows:
\begin{align*}
    h^\cvx_{t,N}(\x)\define\inner{\g_t}{\x},\quad
    h_{t,i}^{\sc}(\x)
    \define \inner{\g_t}{\x}+\frac{\lambda_i}{4}\norm{\x-\x_t}^2,\quad
    h_{t,i}^{\exp}(\x)
    \define \inner{\g_t}{\x}+\frac{\alpha_i}{4}\inner{\g_t}{\x-\x_t}^2.
\end{align*}
Here, $\lambda_i$ and $\alpha_i$ are selected from the candidate pools $\H^{\sc}$ and $\H^{\exp}$, respectively.
For convex loss functions, we run optimistic online gradient descent~(\OOGD) on the surrogate loss function $h^\cvx_{t,N}(\cdot)$ with optimism $\nabla h_{t-1,N}^\cvx(\x_{t-1})$. For strongly convex and exp-concave loss functions, we use \OOGD and optimistic online Newton step~(\OONS), respectively, with surrogate loss functions $h_{t,i}^{\sc}(\cdot)$ and $h_{t,i}^{\exp}(\cdot)$
and corresponding optimisms $\nabla h_{t-1,i}^{\sc}(\x_{t-1,i})$ and $\nabla h_{t-1,i}^{\exp}(\x_{t-1,i})$.  For detailed update rules of differently configured OOMD, we refer readers to \pref{app:universal-bases}.

\begin{algorithm}[t]
    \caption{Universal Gradient-Variation Regret with Corrected Gradient}
    \label{alg:Universal}
    \begin{algorithmic}[1]
        \Require Base learner configurations $\{\mathcal{B}_i\}_{i \in [N]}\define \{\mathcal{B}_i^\scvx\}_{ i\in I^\scvx} \cup \{\mathcal{B}_i^\exp\}_{i \in I^\exp} \cup \{\mathcal{B}^{\cvx}\}$
        \State \textbf{Initialize}: $\M$~---~meta learner \omlprod with $W_{0,i} = \frac{1}{N}$ for all $i \in [N]$ \\
    \makebox[1.8cm]{} $\{\mathcal{B}_i\}_{i \in [N]}$~---~base learners as specified in \pref{app:universal-bases}
        \For{$t=1$ \textbf{to} $T$}
            \State Play $\x_t=\sumN p_{t,i}\x_{t,i}$, incur loss $f_t(\x_t)$, and observe $\nabla f_t(\x_t)$.
            \State Compute the corrected gradient $\g_t=\nabla f_t(\x_t)+\frac{G}{D}(\x_t-\x_{t-1})$.
              \State $\{\mathcal{B}_i\}_{i=1}^N$ update their own decisions to $\{\x_{t+1,i}\}_{i=1}^N$ using  $\g_t$
            \State For each $i\in[N]$, compute $r_{t,i}=\inner{\g_t}{\x_t-\x_{t,i}}$  and set the optimism $\m_{t+1}$ as in \pref{eq:Bregman-optimism}
            \State Send $\r_t$ and $\m_{t+1}$ to $\M$, and obtain $\p_{t+1} \in \Delta_N$ via \pref{eq:omlprod} and \pref{eq:meta-learning-rate}
        \EndFor
    \end{algorithmic}
\end{algorithm}

Following \pref{sec:approach}, we use the corrected gradient $\g_t=\nabla f_t(\x_t)+\gamma(\x_t-\x_{t-1})$ with $\gamma=G/D$.  We denote the strongly convex and exp-concave base learners by $\mathcal{B}_i^\scvx$ and $\mathcal{B}_i^\exp$, respectively, with corresponding index sets $I^\scvx$ and $I^\exp$. The convex base learner is denoted by $\mathcal{B}^\cvx$. We summarize our algorithm in \pref{alg:Universal} and present its universal guarantee in \pref{thm:universal}. The corresponding proof is provided in \pref{app:universal-proof}.
\begin{Theorem}
\label{thm:universal}
Under \pref{assum:boundedness} and \ref{assum:smoothness}, \pref{alg:Universal} achieves the following universal guarantee:
\begin{align*}
    \Reg_T\le
    \begin{cases}
        \O\sbr{\dfrac{1}{\lambda}\log V_T},
            &\text{for $\lambda$-strongly convex losses},\\[2mm]
        \O\sbr{\dfrac{d}{\alpha}\log V_T},
            &\text{for $\alpha$-exp-concave losses},\\[2mm]
        \O\sbr{\sqrt{V_T}},
            &\text{for convex losses}.
    \end{cases}
\end{align*}
\end{Theorem}
We highlight the advantages of \pref{thm:universal} in the following three aspects:

\ding{172}\ \textbf{Regret Bounds.} \pref{thm:universal} improves the regret bound for convex losses in \citet{yan2023universal} from $\O(\sqrt{V_T\log V_T})$ to the optimal rate $\O(\sqrt{V_T})$, while retaining optimal logarithmic regret bounds for strongly convex and exp-concave losses under the same assumptions. We also match the optimal rates of \citet{yan2024simple} while removing their assumption of smoothness on an extended domain.

\ding{173}\ \textbf{Efficiency.} \pref{alg:Universal} uses only $\O(\log T)$ base learners and one gradient query per round. Compared with \citet{yan2023universal}, it reduces the number of base learners from $\O(\log^2 T)$ to $\O(\log T)$. It matches \citet{yan2024simple} in both the number of base learners and gradient queries, while reducing computational overhead: our choice of optimism avoids their one-dimensional fixed-point computation, which requires $\O(\log T)$ iterations per round.

\ding{174}\ \textbf{Analysis.} Our analysis uses a simple corrected gradient to obtain the negative movement term in \pref{eq:dynamic-app} directly. This term allows us to cancel the meta regret without requiring additional negative stability terms in the meta regret bound, avoiding the analysis of cascaded correction terms in \citet{yan2023universal}. It also supports simpler optimism, removing the need to analyze the one-dimensional fixed-point problem in \citet{yan2024simple}.

\section{Application to Fast Convergence in Online Games}
\label{sec:applications-games}
In this section, we apply \pref{thm:universal} to online games and provide a detailed case showing how our negative movement term in \pref{prop:approach} can be used to facilitate fast convergence in games.

We study a two-player zero-sum game of the form $\min_{\x\in\X}\max_{\y\in\Y}f(\x,\y)$.
The $\x$-player minimizes the payoff, while the $\y$-player maximizes it.
To demonstrate universality, we consider two classes of games: bilinear games with $f(\x,\y)=\x^\top A\y$, and strongly convex--strongly concave games in which $f$ is $\lambda$-strongly convex in $\x$ and $\lambda$-strongly concave in $\y$.
At each round $t$, the players choose their strategies simultaneously and then observe their respective loss gradients $\nabla_{\x}f(\x_t,\y_t)$ and $-\nabla_{\y}f(\x_t,\y_t)$. We call the setting honest when both players run the same agreed algorithm in a distributed manner, and dishonest when at least one player deviates from the agreed protocol. In this setting, we use the following standard boundedness and smoothness conditions.

\begin{Assumption}[Strategy sets and payoff gradients]
    \label{ass:game}
    The strategy sets $\X$ and $\Y$ are nonempty compact convex sets with diameter at most $D$, and both partial gradients $\nabla_{\x}f(\x,\y)$ and $\nabla_{\y}f(\x,\y)$ have norm at most $G$.
    For all feasible pairs $(\x,\y)$ and $(\x',\y')$, the payoff satisfies
    \begin{align*}
        \|\nabla_{\x}f(\x,\y)-\nabla_{\x}f(\x',\y')\|
            &\le L(\|\x-\x'\|+\|\y-\y'\|),\\
        \|\nabla_{\y}f(\x,\y)-\nabla_{\y}f(\x',\y')\|
            &\le L(\|\x-\x'\|+\|\y-\y'\|).
     \end{align*}
\end{Assumption}
The pioneering work of \citet{syrgkanis2015fast} showed that, in the honest setting, the sum of players' regrets is bounded by $\O(1)$ when each player runs an algorithm satisfying the RVU property defined in \pref{def:RVU}. This bound enables efficient computation of Nash equilibria. Specifically, for bilinear games, setting the learning rates to $\eta^\x=\eta^\y=\sqrt{\frac{\omega}{L^2\beta}}$ in the RVU bounds yields
\begin{align*}
 \Reg_T^\x+\Reg_T^\y\lesssim {} &  \frac{\alpha}{\eta^\x}+\frac{\alpha}{\eta^\y}+(L^2\beta\eta^\y-\frac{\omega}{\eta^\x})\sumTT\|\x_t-\x_{t-1}\|_1^2\\
 {}&\qquad\qquad+(L^2\beta\eta^\x-\frac{\omega}{\eta^\y})\sumTT\|\y_t-\y_{t-1}\|_1^2\le \O(1),
\end{align*}
where the first step uses the smoothness of the payoff function. In the second step, the negative movement term in the RVU property plays a crucial role in cancellation.

For \pref{alg:Universal}, the negative movement term in \pref{prop:approach} ensures the RVU property, allowing the algorithm to be applied directly to this setting. In both the honest and dishonest settings, \pref{alg:Universal} achieves \emph{state-of-the-art} performance. We provide the proof in \pref{app:applications}.

\begin{table}[!t]
    \centering
    \caption{\small{Comparison of our regret guarantees with existing results. Honest play is measured by the sum of the players' regrets, and dishonest play by each player's individual regret. Both bilinear and strongly convex--strongly concave games are considered.}}
    \label{tab:game}
    \vspace{1mm}
    \begingroup
    \small
    \renewcommand{\arraystretch}{1.5}
    \setlength{\tabcolsep}{4pt}
    \resizebox{0.95\linewidth}{!}{%
    \begin{tabular}{c|c|cccc}
        \hline
        \hline
        & \textbf{Games}
        & \makecell{\citeauthor{syrgkanis2015fast}\\\citeyearpar{syrgkanis2015fast}}
        & \makecell{\citeauthor{zhang2022simple}\\\citeyearpar{zhang2022simple}}
        & \makecell{\citeauthor{yan2023universal}\\\citeyearpar{yan2023universal}}
        & \textbf{Ours}\\
        \hline
        \multirow{2}{*}{\textbf{Honest}}
        & Bilinear & $\O(1)$ & $\O(\sqrt T)$ & $\O(1)$ & $\O(1)$\\
        \cline{2-6}
        & \makecell[c]{Strongly convex--\\strongly concave}
        & $\O(1)$ & $\O(\lambda^{-1}\log T)$ & $\O(1)$ & $\O(1)$\\
        \hline
        \multirow{2}{*}{\textbf{Dishonest}}
        & Bilinear & $\O(\sqrt T)$ & $\O(\sqrt T)$ & $\O(\sqrt{T\log T})$ & $\O(\sqrt T)$\\
        \cline{2-6}
        & \makecell[c]{Strongly convex--\\strongly concave}
        & $\O(\sqrt T)$ & $\O(\lambda^{-1}\log T)$ & $\O(\lambda^{-1}\log T)$ & $\O(\lambda^{-1}\log T)$\\
        \hline
        \hline
    \end{tabular}}
    \endgroup
\end{table}
\begin{Theorem}
    \label{thm:game}
    Under \pref{ass:game}, for both bilinear and strongly convex--strongly concave games, \pref{alg:Universal} enjoys $\O(1)$ regret summation in the honest case:
    \begin{equation*}
        \Reg_T^{\x}+\Reg_T^{\y}\le\O(1).
    \end{equation*}
In the dishonest case, the algorithm achieves individual regret bounds of $\O(\sqrt T)$ for bilinear games and $\O(\lambda^{-1}\log T)$ for strongly convex--strongly concave games.
\end{Theorem}
\pref{tab:game} compares our guarantees with prior work.
Under honest play, \pref{thm:game} matches the $\O(1)$ regret sums of \citet{yan2023universal} for both game classes, improving upon the corresponding bounds of \citet{zhang2022simple}.
Under dishonest play, it matches the individual regret bounds of \citet{zhang2022simple} for both game classes and improves the dishonest-play bound of \citet{yan2023universal} from $\O(\sqrt{T\log T})$ to $\O(\sqrt T)$ for bilinear games.

\section{Conclusion}
\label{sec:Summary}
We propose a \emph{simple} and \emph{unified} method for obtaining negative terms in online regret bounds. These terms are key to establishing gradient-variation guarantees and fast convergence in games. In the non-stationary setting, we establish a \emph{black-box reduction} from optimistic regret bounds to gradient-variation bounds. This reduction enables a broad class of algorithms to recover the gradient-variation dynamic regret guarantee $\O(\sqrt{(1+V_T+P_T)(1+P_T)})$ established by \citet{zhao2024adaptivity}.
Our method yields a \emph{simple} and \emph{efficient} universal algorithm with regret bounds of $\O(\sqrt{V_T})$, $\O(\frac{d}{\alpha}\log V_T)$, and $\O(\frac{1}{\lambda}\log V_T)$ for convex, $\alpha$-exp-concave, and $\lambda$-strongly convex losses, respectively. It matches the \emph{optimal} rates of \citet{yan2024simple} without their extended smoothness assumption, uses only $\O(\log T)$ base learners and one gradient query per round, and incurs lower computational cost per round. We further apply the resulting universal guarantee to online games and improve on existing guarantees~\citep{zhang2022simple,yan2023universal}.

\newpage
\section*{AI-use Statement}
The authors developed the idea of introducing a negative movement term with assistance from GPT-6 Astra.
Building on this idea, they developed new algorithms for gradient-variation non-stationary online learning and obtained improved results for gradient-variation universal online learning.
GPT-6 Astra was also used to assist with reviewing and polishing the manuscript.
The authors take full responsibility for the content and correctness of the paper.
\bibliography{utility/online}
\newpage

\appendix

\section{Omitted Details in \pref{sec:approach}}
\label{app:approach}
In this section, we provide the proof of \pref{prop:approach} in \pref{app:movement-transfer} and its extension to general Bregman divergences in \pref{app:bregman-movement}. 
\subsection{Proof of \pref{prop:approach}}
\label{app:movement-transfer}
\begin{proof}
Recall that $\v_t\define \nabla f_t(\x_t)+\gamma (\x_t-\x_{t-1})$. We have
\begin{align}
  &\sum_{t=1}^T\inner{\nabla f_t(\x_t)}{\x_t-\u_t}\notag
  =\sum_{t=1}^T\inner{\v_t}{\x_t-\u_t}
  -\gamma\sum_{t=1}^T\inner{\x_t-\x_{t-1}}{\x_t-\u_t}\notag\\
  ={}&\sum_{t=1}^T\inner{\v_t}{\x_t-\u_t}
  -\frac\gamma2\sumTT\norm{\x_t-\x_{t-1}}^2
  -\frac\gamma2\sum_{t=1}^T
  \left(\norm{\x_t-\u_t}^2-\norm{\x_{t-1}-\u_t}^2\right)\notag\\
  \le{}&\sum_{t=1}^T\inner{\v_t}{\x_t-\u_t}
  -\frac\gamma2\sumTT\norm{\x_t-\x_{t-1}}^2+\frac{\gamma D^2}{2}
  -\frac\gamma2\sum_{t=2}^T
  \left(\norm{\x_{t-1}-\u_{t-1}}^2-\norm{\x_{t-1}-\u_t}^2\right)\notag\\
  \le{}&\sum_{t=1}^T\inner{\v_t}{\x_t-\u_t}
  -\frac\gamma2\sumTT\norm{\x_t-\x_{t-1}}^2+\frac{\gamma D^2}{2}+\gamma DP_T,\label{eq:prop1}
\end{align}
where the second step follows from the three-point identity $2\inner{\a-\b}{\a-\c}
  =\norm{\a-\c}^2-\norm{\b-\c}^2+\norm{\a-\b}^2$, for $\a,\b,\c\in\R^d$, the third step uses the boundedness assumption and the last inequality follows from $\|\a\|^2-\|\b\|^2\le \|\a+\b\|\cdot\|\a-\b\|$ for any $\a,\b\in\R^d$.
\end{proof}
\subsection{Extension to Bregman Movement}
\label{app:bregman-movement}
\begin{Proposition}
Under \pref{assum:boundedness}, let $f_t$ be convex for every $t\in[T]$ and let $\psi$ be differentiable, convex, and $L^\prime$-smooth on $\X$. Set $\x_0=\x_1$ and let $\v_t\define\nabla f_t(\x_t)+\gamma\sbr{\nabla\psi(\x_t)-\nabla\psi(\x_{t-1})}$ be the corrected gradient.
For any $\gamma>0$ and every decision sequence $\x_{1:T}\in\X^T$ and comparator sequence $\u_{1:T}\in\X^T$, we have
\begin{align*}
   \sumT \inner{\nabla f_t(\x_t)}{\x_t-\u_t}
   \le {}&\sumT\inner{\v_t}{\x_t-\u_t}
    -\gamma\sumTT\D_{\psi}(\x_t,\x_{t-1})+\frac{\gamma L^\prime D^2}{2}+\gamma L^\prime DP_T,
\end{align*}
where $P_T\define\sum_{t=2}^T\norm{\u_t-\u_{t-1}}$ is the comparator path length.    
\end{Proposition}
\begin{proof}
Recall that $\v_t\define\nabla f_t(\x_t)+\gamma\sbr{\nabla\psi(\x_t)-\nabla\psi(\x_{t-1})}$ and $\x_0=\x_1$. We have
\begin{align*}
  &\sumT\inner{\nabla f_t(\x_t)}{\x_t-\u_t}
  =\sumT\inner{\v_t}{\x_t-\u_t}-\gamma\sumT\inner{\nabla\psi(\x_t)-\nabla\psi(\x_{t-1})}{\x_t-\u_t}\\
  ={}&\sumT\inner{\v_t}{\x_t-\u_t}
    -\gamma\sumTT\D_\psi(\x_t,\x_{t-1})-\gamma\sumT\sbr{\D_\psi(\u_t,\x_t)-\D_\psi(\u_t,\x_{t-1})}\\
  \le{}&\sumT\inner{\v_t}{\x_t-\u_t}
    -\gamma\sumTT\D_\psi(\x_t,\x_{t-1})
    +\frac{\gamma L^\prime D^2}{2}-\gamma\sumTT\sbr{\D_\psi(\u_{t-1},\x_{t-1})-\D_\psi(\u_t,\x_{t-1})}\\
  \le{}&\sumT\inner{\v_t}{\x_t-\u_t}
    -\gamma\sumTT\D_\psi(\x_t,\x_{t-1})+\frac{\gamma L^\prime D^2}{2}+\gamma L^\prime DP_T,
\end{align*}
where the second equality follows from the Bregman three-point identity
$\inner{\nabla\psi(\a)-\nabla\psi(\b)}{\a-\c}
  =\D_\psi(\c,\a)-\D_\psi(\c,\b)+\D_\psi(\a,\b)$ for $\a,\b,\c\in\R^d$, the third inequality uses smoothness, and the last inequality uses $
|\D_\psi(\u,\x)-\D_\psi(\u',\x)|\le L^\prime D\norm{\u-\u'}$ for all $\u,\u',\x\in\X$.
\end{proof}

\section{Omitted Proofs in \pref{sec:nonstationary}}
\label{app:nonstationary}
In this section, we provide the omitted proofs for \pref{thm:optimistic-gradient-variation-reduction} and \pref{cor:dynamic-gradient-variation}.

\subsection{Proof for \pref{thm:optimistic-gradient-variation-reduction}}
\label{app:optimistic-gradient-variation-reduction}
\begin{proof}
For the first choice of $\m_t$ in \pref{eq:mtchoice}, we have
\begin{align*}
    &\sumT\norm{\g_t-\m_t}^2=\sumT \|\nabla f_t(\x_t)-\nabla f_{t-1}(\x_{t-1})+L(\x_t-\x_{t-1})-L(\x_{t-1}-\x_{t-2})\|^2\notag\\
    \le {}&G^2+3\sumTT\norm{\nabla f_t(\x_t)-\nabla f_{t-1}(\x_{t-1})}^2+12L^2\sumTT\norm{\x_t-\x_{t-1}}^2\notag\\
    \le {}&G^2+6V_T+18L^2\sumTT\norm{\x_t-\x_{t-1}}^2.
\end{align*}

For the second choice of $\m_t$ in \pref{eq:mtchoice}, we have
\begin{align*}
    &\sumT\norm{\g_t-\m_t}^2=\sumT \|\nabla f_t(\x_t)-\nabla f_{t-1}(\x_{t-1})+L(\x_t-\x_{t-1})\|^2\notag\\
    \le {}&G^2+2\sumTT\norm{\nabla f_t(\x_t)-\nabla f_{t-1}(\x_{t-1})}^2+2L^2\sumTT\norm{\x_t-\x_{t-1}}^2\notag\\
    \le {}&G^2+4V_T+6L^2\sumTT\norm{\x_t-\x_{t-1}}^2.
\end{align*}

For both choices, plugging the bound of $\sumT\norm{\g_t-\m_t}^2$ into \pref{eq:dynamic-optimistic}, we have
\begin{align*}
   \sumT  \inner{\g_t}{\x_t-\u_t}
    \lesssim \sqrt{(1+P_T)\sbr{1+P_T+V_T+L^2\sumTT\norm{\x_t-\x_{t-1}}^2}}.
\end{align*}
Plugging it back to \pref{eq:prop1} with the choice of $\gamma=L$, we have
\begin{align*}
    &\DReg_T(\u_{1:T})\le\sumT\inner{\nabla f_t(\x_t)}{\x_t-\u_t}\notag\\
    \le{}&\sum_{t=1}^T\inner{\g_t}{\x_t-\u_t}-\frac{L}{2}\sumTT\norm{\x_t-\x_{t-1}}^2+\frac{L D^2}{2}+L DP_T,\notag\\
    \le{}& \O(\sqrt{(1+P_T)\sbr{1+P_T+V_T}})+\O\sbr{\sqrt{(1+P_T)\sumTT\norm{\x_t-\x_{t-1}}^2}}-\frac{L}{2}\sumTT\norm{\x_t-\x_{t-1}}^2\\
    \le{}&\O(\sqrt{(1+P_T)\sbr{1+P_T+V_T}}),
\end{align*}
where the last inequality follows from $2\sqrt{ab}-a\le b$ for $a,b\ge0$.
\end{proof}

\subsection{Proof of \pref{cor:dynamic-gradient-variation}}
\label{app:dynamic-gradient-variation}
\begin{proof}
Setting $\lambda=0$ in Theorem 9 of \citet{zhao2024adaptivity} shows that \Swordo satisfies \pref{eq:dynamic-optimistic}. Applying \pref{thm:optimistic-gradient-variation-reduction} then completes the proof.
\end{proof}

\section{Omitted Details in \pref{sec:universality}}
\label{app:universality}
In this section, we provide the omitted details of the base learners in \pref{alg:Universal} in \pref{app:universal-bases} and prove \pref{thm:universal} in \pref{app:universal-proof}.

\subsection{Base Learners of \pref{alg:Universal}}
\label{app:universal-bases}
Following the conversion stated after \pref{def:expconcave}, we use the effective exp-concavity coefficient $\alpha=\frac{1}{2}\min\{1/(4GD),\beta\}$, where $\exp(-\beta f_t(\cdot))$ is concave for every $t\in[T]$.
In particular, $\alpha GD\le1/8$.
For the analysis, it suffices to consider $\lambda\in[1/T,1]$ and $\alpha\in[1/T,\min\{1,1/(8GD)\}]$.
If either parameter is below $1/T$, the corresponding logarithmic regret bound does not improve on the trivial $\O(T)$ bound, while the general convex regret bound remains applicable.
Conversely, if $\lambda,\alpha > 1$, they can be treated as $\lambda = 1$ and $\alpha = 1$, which only worsens the regret by an ignorable constant factor.

Then, we construct the candidate pools $\H^{\scvx}$ and $\H^{\exp}$ for the unknown parameters $\lambda$ and $\alpha$  respectively, using a geometric grid:
\begin{equation}
    \label{eq:candidate-pool}
    \H^{\exp} = \H^{\scvx} \triangleq \left\{ \frac{1}{T}, \frac{2}{T}, \frac{2^2}{T}, \cdots, \frac{2^{n-1}}{T} \right\},
\end{equation}
where $n = \ceil{\log_2 T} + 1 = \O(\log T)$ is the number of candidates in each pool. One base learner is instantiated for each candidate $\lambda_i$ or $\alpha_i$, with the final base learner, indexed by $N$, assigned to convex losses. Thus, there are $N=2n+1=\O(\log T)$ base learners in total. We define the best base learner as the one with the right guess of the curvature type and the closest guess of the curvature coefficient.
For example, suppose the online functions are $\alpha$-exp-concave (while this is unknown to the online learner). The best base learner, indexed by $\is$, uses a candidate parameter $\alpha_\is$ satisfying $\alpha_\is \le \alpha \le 2\alpha_\is$.

The base learners below use the corrected gradient $\g_t$ from \pref{alg:Universal}, with $\norm{\g_t}\le H$ as in the main text.
Under \pref{assum:boundedness}, $\norm{\g_t}\le G+\gamma D$, which equals $2G$ when $\gamma=G/D$.
Let $I^\exp\define\{1,\ldots,n\}$ and $I^\scvx\define\{n+1,\ldots,2n\}$ be the two curved index sets.
Initialize $\x_0=\x_1$, $\g_0=\zero$, and $\xh_{1,i}=\x_{0,i}=\x_1$ for every $i\in[N]$, with $h_{0,i}\equiv0$.
We write $h_{t,i}$ for the surrogate assigned to base $i$; in particular, $h_{t,N}=h_{t,N}^\cvx$ for the convex base.

For \emph{$\alpha$-exp-concave} functions, we define a group of base learners $\{\Bc_i^\exp\}_{i\in I^\exp}$, where the $i$-th base learner runs the algorithm below:
\begin{equation}
    \label{eq:base-exp}
   \begin{aligned}
        \x_{t,i} = {}&\argmin_{\x\in\X}
        \bbr{\inner{\nabla h_{t-1,i}^{\exp}(\x_{t-1,i})}{\x}
            +\D_{\psi_{t,i}}(\x,\xh_{t,i})},\\
    \xh_{t+1,i} = {}&\argmin_{\x\in\X}
        \bbr{\inner{\nabla h_{t,i}^{\exp}(\x_{t,i})}{\x}
            +\D_{\psi_{t,i}}(\x,\xh_{t,i})},
   \end{aligned} 
\end{equation}

where $\psi_{t,i}(\x) \define \frac{1}{2}\x^\top U_{t,i}\x$, $U_{t,i} = \sbr{1+\frac{\alpha_iH^2}{2}}I_d+\frac{\alpha_i}{8}\sum_{s=1}^{t-1}\nabla h_{s,i}^{\exp}(\x_{s,i})\nabla h_{s,i}^{\exp}(\x_{s,i})^\top$,
$\alpha_i$ is the corresponding element in $\H^\exp$, and $h_{t,i}^{\exp}(\cdot)$ is a surrogate loss function for $\Bc_i^\exp$, defined as
\begin{align*}
    h_{t,i}^{\exp}(\x) \define {}&\inner{\g_t}{\x}
        +\frac{\alpha_i}{4}\inner{\g_t}{\x-\x_t}^2.
\end{align*}
Similarly, for \emph{$\lambda$-strongly convex} functions, we define a group of base learners $\{\Bc_i^\scvx\}_{i\in I^\scvx}$, where the $i$-th base learner runs the algorithm below:
\begin{align}
    \label{eq:base-scvx}
    \x_{t,i} = \Pi_\X\mbr{\xh_{t,i}
        -\eta_{t,i}\nabla h_{t-1,i}^{\sc}(\x_{t-1,i})},\quad
    \xh_{t+1,i} = \Pi_\X\mbr{\xh_{t,i}
        -\eta_{t,i}\nabla h_{t,i}^{\sc}(\x_{t,i})},
\end{align}
where $\eta_{t,i}=4/(2+\lambda_i t)$, $\lambda_i$ is the corresponding element in $\H^\scvx$, and $h_{t,i}^{\sc}(\cdot)$ is a surrogate loss function for $\Bc_i^\scvx$, defined as
\begin{align*}
    h_{t,i}^{\sc}(\x) \define {}&\inner{\g_t}{\x}
        +\frac{\lambda_i}{4}\norm{\x-\x_t}^2.
\end{align*}
For \emph{convex} functions, we only have to define one base learner $\Bc^\cvx$, which updates as
\begin{align}
     \label{eq:base-cvx}
    \x_{t,N} = \Pi_\X\mbr{\xh_{t,N}
        -\eta_{t,N}\g_{t-1}},\quad
    \xh_{t+1,N} = \Pi_\X\mbr{\xh_{t,N}
        -\eta_{t,N}\g_t},
\end{align}
where $\eta_{t,N} = \min\bbr{
        D/{\sqrt{1+\sum_{s=2}^{t-1}
            \norm{\g_s-\g_{s-1}}^2}},1}$.
Finally, we summarize the configurations of base learners. Specifically, we deploy
\begin{align*}
    \{\Bc_i\}_{i\in[N]} \define {}&
        \{\Bc_i^\exp\}_{i\in I^\exp}
        \cup\{\Bc_i^\scvx\}_{i\in I^\scvx}
        \cup\{\Bc^\cvx\}
\end{align*}
as the total set of base learners.

\subsection{Proof of \pref{thm:universal}}
\label{app:universal-proof}
\begin{proof}
Here, we consider any comparator $\u\in\X$.
We first decompose the regret for the three curvature types, then analyze the meta and base regret separately, and finally combine their bounds. 
As in \pref{alg:Universal}, we use the corrected gradient
$\g_t=\nabla f_t(\x_t)+\gamma(\x_t-\x_{t-1})$, with $\gamma=G/D$ and $\norm{\g_t}\le H=2G$.
For simplicity, we write
\begin{align*}
    S_T\define\sumTT\norm{\x_t-\x_{t-1}}^2,\qquad
    \Vt_T\define\sumTT\norm{\g_t-\g_{t-1}}^2,
    \qquad K_\gamma\define3L^2+12\gamma^2.
\end{align*}
Here $\Vt_T$ is the variation of the corrected gradients; $\Vb_T$ in \pref{sec:preliminary} denotes the variation of the original loss gradients.
Under \pref{assum:smoothness}, the same decomposition as in \pref{eq:method-movement-decomposition} gives
\begin{align}
    \Vt_T\le{}& 3\sum_{t=2}^T\norm{\nabla f_t(\x_t)-\nabla f_{t-1}(\x_t)}^2
  +3\sum_{t=2}^T\norm{\nabla f_{t-1}(\x_t)-\nabla f_{t-1}(\x_{t-1})}^2\notag\\
  &\qquad+3\gamma^2\sum_{t=2}^T\norm{\x_t-2\x_{t-1}+\x_{t-2}}^2\notag\\
  &\le 3V_T+3L^2\sum_{t=2}^T\norm{\x_t-\x_{t-1}}^2+6\gamma^2\sum_{t=2}^T
  \left(\norm{\x_t-\x_{t-1}}^2+\norm{\x_{t-1}-\x_{t-2}}^2\right)\notag\\
  &\le 3V_T+(3L^2+12\gamma^2)S_T.
    \label{eq:universal-corrected-variation}
\end{align}

\paragraph{Regret Decomposition.}
For each base index $i\in[N]$, put $r_{t,i}=\inner{\g_t}{\x_t-\x_{t,i}}$. The surrogates $h_{t,i}$ are those explicitly defined in \pref{app:universal-bases}.

For $\lambda$-strongly convex functions, choose an index $\is\in I^\scvx$ whose guess $\lambda_{\is}$ satisfies $\lambda_{\is}\le\lambda\le2\lambda_{\is}$. Strong convexity and \pref{eq:prop1} imply
\begin{align}
\Reg_T \le {}&\sumT\left(\inner{\g_t}{\x_t-\u}
-\frac\lambda2\norm{\x_t-\u}^2\right)
-\frac\gamma2S_T+\frac{\gamma D^2}{2}\nonumber\\*
\le {}&\underbrace{\sumT\left(r_{t,\is}-\frac{\lambda_{\is}}{4}\norm{\x_t-\x_{t,\is}}^2\right)}_{\MetaReg^{\scvx}}
+\underbrace{\sumT\bigl(h_{t,\is}^{\sc}(\x_{t,\is})-h_{t,\is}^{\sc}(\u)\bigr)}_{\BaseReg^{\scvx}}-\frac\gamma2S_T+\frac{\gamma D^2}{2}. 
\label{eq:decomposition-sc}
\end{align}
Here $h_{t,\is}^{\sc}(\x)=\inner{\g_t}{\x}+(\lambda_{\is}/4)\norm{\x-\x_t}^2$. 

For $\alpha$-exp-concave functions, choose $\is\in I^\exp$ with $\alpha_\is\le\alpha\le2\alpha_\is$ and transfer the directional curvature from the original gradient $\nabla f_t(\x_t)$ to the corrected gradient $\g_t$. The inequality $(q-b)^2\ge q^2/2-b^2$ gives
\begin{align}
-\frac\alpha2\sumT\inner{\nabla f_t(\x_t)}{\x_t-\u}^2\le {}&-\frac\alpha4\sumT\inner{\g_t}{\x_t-\u}^2
+\frac{\alpha\gamma^2D^2}{2}S_T.
\label{eq:curvature-transfer}
\end{align}
Since $\alpha\gamma D^2=\alpha GD\le1/8$, combining this inequality with \pref{eq:prop1} gives
\begin{align}
\Reg_T\le {}&\sumT\left(\inner{\g_t}{\x_t-\u}
-\frac\alpha4\inner{\g_t}{\x_t-\u}^2\right)
-\frac\gamma4S_T+\frac{\gamma D^2}{2}\nonumber\\
\le{}&\underbrace{\sumT\left(r_{t,\is}-\frac{\alpha_{\is}}{4}r_{t,\is}^2\right)}_{\MetaReg^{\exp}}
+\underbrace{\sumT\bigl(h_{t,\is}^{\exp}(\x_{t,\is})-h_{t,\is}^{\exp}(\u)\bigr)}_{\BaseReg^{\exp}}-\frac\gamma4S_T+\frac{\gamma D^2}{2},
\label{eq:decomposition-exp}
\end{align}
where $h_{t,\is}^{\exp}(\x)\define\inner{\g_t}{\x}+(\alpha_{\is}/4)\inner{\g_t}{\x-\x_t}^2$.

For convex functions, using the convex base indexed by $N$, convexity and \pref{eq:prop1} give
\begin{align}
\Reg_T \le {}&\sumT\inner{\g_t}{\x_t-\u}
-\frac\gamma2S_T+\frac{\gamma D^2}{2}\nonumber\\*
={} &\underbrace{\sumT r_{t,N}}_{\MetaReg^\cvx}
+\underbrace{\sumT\bigl(h_{t,N}^{\cvx}(\x_{t,N})-h_{t,N}^{\cvx}(\u)\bigr)}_{\BaseReg^\cvx}
-\frac\gamma2S_T+\frac{\gamma D^2}{2}.
\label{eq:decomposition-convex}
\end{align}

\paragraph{Meta Regret Analysis.}
We adopt \omlprod~\citep{wei2016tracking} as the meta learner and recall its optimistic second-order regret guarantee below.
\begin{Lemma}[Theorem 3.4 of \citet{wei2016tracking}]
\label{lem:optimistic-mlprod}
Denote by $\p_t\in\Delta_N$ the algorithm's weights, $\ellb_t\in[0,1]^N$ the loss vector, and $\widetilde m_{t,i}\in[-1,1]$ the optimism available before observing $\ellb_t$.
The regret of \omlprod with respect to any base learner $i\in[N]$ satisfies
\begin{align*}
\sumT\widetilde r_{t,i}\le {}& C_0\sqrt{1+\sum_{t=1}^T(\widetilde r_{t,i}-\widetilde m_{t,i})^2}+C_2,
\end{align*}
where $\widetilde r_{t,i}=\inner{\p_t}{\ellb_t}-\ell_{t,i}$, $C_0 = \sqrt{\log N} + \log (1 + \frac{N}{e} (1 + \log (T+1))) / \sqrt{\log N}$, and $C_2 = \frac{1}{4}(\log N + \log (1 + \frac{N}{e} (1 + \log (T+1)))) + 2 \sqrt{\log N} + 16 \log N$.
\end{Lemma}
To apply \pref{lem:optimistic-mlprod}, we define $\ell_{t,i}\define\frac{1}{2}+\frac{\inner{\g_t}{\x_{t,i}-\x_1}}{2HD}$ and rewrite \pref{eq:omlprod} and \pref{eq:meta-learning-rate} in the following normalized form:
\begin{align*}
    \widetilde r_{t,i}\define{}&\frac{r_{t,i}}{2HD}
        =\inner{\p_t}{\ellb_t}-\ell_{t,i},\quad
    \widetilde m_{t,i}\define\frac{m_{t,i}}{2HD},\\
    \widetilde\epsilon_{t,i}={} &2HD\,\epsilon_{t,i}=\min\bbr{\frac{1}{4},
        \sqrt{\frac{\ln N}{1+\sum_{s=1}^t
            (\widetilde r_{s,i}-\widetilde m_{s,i})^2}}}.
\end{align*}
This normalization leaves the meta learner's weight updates and predictions unchanged. Since $\norm{\g_t}\le H$ and $\norm{\x_{t,i}-\x_1}\le D$, we have $\ell_{t,i}\in[0,1]$ and $\widetilde m_{t,i}=m_{t,i}/(2HD)\in[-1/2,1/2]$. Thus \pref{lem:optimistic-mlprod} yields
\begin{align}
\sumT r_{t,i}\le {}& C_0\sqrt{\sumT(r_{t,i}-m_{t,i})^2}+C_1HD,
\label{eq:meta-guarantee}
\end{align}
where $C_1=2C_0+2C_2=\O(\log N)$.

For \emph{$\lambda$-strongly convex} functions, according to \pref{eq:Bregman-optimism}, we have $m_{t,\is}=0$ for the selected strongly convex base learner $\is$.
By \pref{lem:optimistic-mlprod}, the meta regret in \pref{eq:decomposition-sc} can be bounded as
\begin{align}
\MetaReg^{\scvx} \le {}& C_0H\sqrt{\sumT\norm{\x_t-\x_{t,\is}}^2}+C_1HD-\frac{\lambda_{\is}}{4}\sumT\norm{\x_t-\x_{t,\is}}^2\nonumber\\*
\le {}& C_1HD+\frac{2C_0^2H^2}{\lambda_{\is}}-\frac{\lambda_{\is}}{8}\sumT\norm{\x_t-\x_{t,\is}}^2.
\label{eq:meta-sc-bound}
\end{align}
Here the first inequality uses $r_{t,\is}^2\le H^2\norm{\x_t-\x_{t,\is}}^2$, and the second uses the AM--GM inequality.

For \emph{exp-concave} functions, we similarly have $m_{t,\is}=0$ for the selected exp-concave base learner $\is$.
By \pref{lem:optimistic-mlprod}, the meta regret in \pref{eq:decomposition-exp} satisfies
\begin{align}
\MetaReg^{\exp} \le {}& C_0\sqrt{\sumT r_{t,\is}^2}+C_1HD-\frac{\alpha_{\is}}{4}\sumT r_{t,\is}^2\nonumber\\*
\le {}& C_1HD+\frac{2C_0^2}{\alpha_{\is}}-\frac{\alpha_{\is}}{8}\sumT r_{t,\is}^2.
\label{eq:meta-exp-bound}
\end{align}
The last step again follows from the AM--GM inequality.
The remaining negative terms in these two bounds will cancel the surrogate-dependent terms in the base regret.

For \emph{convex} functions, \pref{eq:Bregman-optimism} gives $m_{t,N}=\inner{\g_{t-1}}{\x_{t-1}-\x_{t,N}}$.
Consequently, the prediction error decomposes as
\begin{align*}
r_{t,N}-m_{t,N}
=\inner{\g_t-\g_{t-1}}{\x_t-\x_{t,N}}
+\inner{\g_{t-1}}{\x_t-\x_{t-1}}.
\end{align*}
Using $\norm{\x_t-\x_{t,N}}\le D$, $\norm{\g_{t-1}}\le H$, and $\g_0=\zero$, we obtain
\begin{align}
\sumT(r_{t,N}-m_{t,N})^2
\le H^2D^2+2D^2\Vt_T+2H^2S_T.
\label{eq:convex-prediction-error}
\end{align}
By \pref{lem:optimistic-mlprod}, the meta regret in \pref{eq:decomposition-convex} is therefore bounded by
\begin{align}
\MetaReg^\cvx \le {}& C_0\sqrt{2D^2\Vt_T+2H^2S_T+H^2D^2}+C_1HD\nonumber\\*
\le {}&\sqrt2C_0D\sqrt{\Vt_T}+\sqrt2C_0H\sqrt{S_T}+(C_0+C_1)HD.
\label{eq:meta-convex-bound}
\end{align}
The additional $\sqrt{S_T}$ term is absorbed by the negative movement term in \pref{prop:approach}.

\paragraph{Base Regret Analysis.}
For \emph{$\lambda$-strongly convex} functions, when using the update rule~\pref{eq:base-scvx}, following the OOMD analysis used for Eq.~(C.9) of \citet{yan2024simple}, the base regret can be bounded as
\begin{align}
    \BaseReg^\scvx\le{}&\O\sbr{\frac{1}{\lambda_\is}\log\sbr{1+\lambda_\is\sumTT\norm{\nabla h_{t,\is}(\x_{t,\is})-\nabla h_{t-1,\is}(\x_{t-1,\is})}^2}}.
    \label{eq:universal-sc-base-raw}
\end{align}
Next, we analyze the empirical gradient variation of the surrogate function $h_{t,\is}^\sc(\cdot)$:
\begin{align}
    &\sumTT\norm{\nabla h_{t,\is}\!(\x_{t,\is})-\nabla h_{t-1,\is}\!(\x_{t-1,\is})}^2=\sumTT\!\norm{\g_t-\g_{t-1}+\tfrac{\lambda_\is}{2}\sbr{\x_{t,\is}-\x_t-\x_{t-1,\is}+\x_{t-1}}}^2\notag\\
    &\le3\Vt_T+\frac{3\lambda_\is^2}{2}\sumT\norm{\x_t-\x_{t,\is}}^2\le 9V_T+(9L^2+36\gamma^2)S_T+\frac{3\lambda_\is^2}{2}\sumT\norm{\x_t-\x_{t,\is}}^2,
    \label{eq:universal-sc-base-variation}
\end{align}
where the last inequality uses \pref{eq:universal-corrected-variation}.
Plugging the surrogate's empirical gradient variation into \pref{eq:universal-sc-base-raw}, we obtain
\begin{align}
    \BaseReg^\scvx\le{}&\O\sbr{\frac{1}{\lambda_\is}\log\sbr{1+\lambda_\is\sbr{V_T+S_T+\lambda_\is^2\sumT\norm{\x_t-\x_{t,\is}}^2}}}.
    \label{eq:universal-sc-base-bound}
\end{align}

For \emph{$\alpha$-exp-concave} functions, when using the update rule~\pref{eq:base-exp}, following the OOMD analysis used for Eq.~(C.7) of \citet{yan2024simple}, the base regret can be bounded as
\begin{align}
    \BaseReg^\exp\le{}&\O\sbr{\frac{d}{\alpha_\is}\log\sbr{1+\frac{\alpha_\is}{d}\sumTT\norm{\nabla h_{t,\is}(\x_{t,\is})-\nabla h_{t-1,\is}(\x_{t-1,\is})}^2}}+\O(1).
    \label{eq:universal-exp-base-raw}
\end{align}
Next, we analyze the empirical gradient variation of the surrogate function $h_{t,\is}^\exp(\cdot)$:
\begin{align}
    &\sumTT\norm{\nabla h_{t,\is}(\x_{t,\is})-\nabla h_{t-1,\is}(\x_{t-1,\is})}^2=\sumTT\norm{\g_t-\g_{t-1}-\frac{\alpha_\is}{2}\sbr{\g_t r_{t,\is}-\g_{t-1}r_{t-1,\is}}}^2\notag\\
    &\le3\Vt_T+\frac{3\alpha_\is^2H^2}{2}\sumT r_{t,\is}^2
    \le9V_T+(9L^2+36\gamma^2)S_T+\frac{3\alpha_\is^2H^2}{2}\sumT r_{t,\is}^2,
    \label{eq:universal-exp-base-variation}
\end{align}
where the equality uses $\nabla h_{t,\is}(\x_{t,\is})=\g_t(1-\alpha_\is r_{t,\is}/2)$, the first inequality uses the squared triangle inequality, a shift of the summation index, and $\norm{\g_t}\le H$, and the last uses \pref{eq:universal-corrected-variation}.
Plugging the surrogate's empirical gradient variation into \pref{eq:universal-exp-base-raw}, we obtain
\begin{align}
    \BaseReg^\exp\le{}&\O\sbr{\frac{d}{\alpha_\is}\log\sbr{1+\frac{\alpha_\is}{d}\sbr{V_T+S_T+\alpha_\is^2\sumT r_{t,\is}^2}}}.
    \label{eq:universal-exp-base-bound}
\end{align}

For \emph{convex} functions, when using the update rule~\pref{eq:base-cvx}, following the OOMD analysis in the proof of Lemma 10 of \citet{yan2023universal}, the base regret can be bounded as
\begin{align}
    \BaseReg^\cvx\le{}&\O\sbr{\sqrt{1+\sumTT\norm{\nabla h_{t,N}(\x_{t,N})-\nabla h_{t-1,N}(\x_{t-1,N})}^2}}.
    \label{eq:universal-convex-base-raw}
\end{align}
Next, we analyze the empirical gradient variation of the surrogate function $h_{t,N}^\cvx(\cdot)$:
\begin{align}
    \sumTT\norm{\nabla h_{t,N}(\x_{t,N})-\nabla h_{t-1,N}(\x_{t-1,N})}^2=\Vt_T
    \le3V_T+(3L^2+12\gamma^2)S_T,
    \label{eq:universal-convex-base-variation}
\end{align}
where the equality uses $\nabla h_{t,N}(\x)=\g_t$ and the expression for $\Vt_T$, and the inequality uses \pref{eq:universal-corrected-variation}.
Plugging the surrogate's empirical gradient variation into \pref{eq:universal-convex-base-raw}, we obtain
\begin{align}
    \BaseReg^\cvx\le{}&\O\sbr{\sqrt{1+V_T+S_T}}\le\O\sbr{\sqrt{1+V_T}+\sqrt{S_T}}.
    \label{eq:universal-convex-base-bound}
\end{align}

\paragraph{Overall Regret Analysis.}
For \emph{$\lambda$-strongly convex} functions, combining \pref{eq:decomposition-sc}, \pref{eq:meta-sc-bound}, and \pref{eq:universal-sc-base-bound}, we obtain
\begin{align*}
    \Reg_T\le{}&\O\sbr{\frac{1}{\lambda_\is}\log\sbr{2+V_T+S_T+\lambda_\is^2\sumT\norm{\x_t-\x_{t,\is}}^2}}-\frac{\lambda_\is}{8}\sumT\norm{\x_t-\x_{t,\is}}^2-\frac\gamma2S_T\\
    \le{}&\O\sbr{\frac{1}{\lambda}\log V_T},
\end{align*}
where the last inequality uses $\log(1+x)\le x$ and $\lambda_\is\le\lambda\le2\lambda_\is$.

For \emph{$\alpha$-exp-concave} functions, combining \pref{eq:decomposition-exp}, \pref{eq:meta-exp-bound}, and \pref{eq:universal-exp-base-bound}, we obtain
\begin{align*}
    \Reg_T\le{}&\O\sbr{\frac{d}{\alpha_\is}\log\sbr{2+\frac{\alpha_\is}{d}\sbr{V_T+S_T+\alpha_\is^2\sumT r_{t,\is}^2}}}-\frac{\alpha_\is}{8}\sumT r_{t,\is}^2-\frac\gamma4S_T\\
    \le{}&\O\sbr{\frac{d}{\alpha}\log V_T},
\end{align*}
where the last inequality uses $\log(1+x)\le x$  and $\alpha_\is\le\alpha\le2\alpha_\is$.

For \emph{convex} functions, \pref{eq:meta-convex-bound} and \pref{eq:universal-convex-base-raw} give the empirical guarantee
\begin{align}
    \sumT\inner{\g_t}{\x_t-\u}\le{}&\O\sbr{\sqrt{\Vt_T}+\sqrt{S_T}}.
    \label{eq:universal-convex-empirical}
\end{align}
Combining this with \pref{eq:decomposition-convex} and \pref{eq:universal-corrected-variation}, and absorbing $\sqrt{S_T}$ by Young's inequality, gives
\begin{align}
    \Reg_T\le{}&\O\sbr{\sqrt{V_T}+\sqrt{S_T}}-\frac\gamma2S_T\le\O\sbr{\sqrt{V_T}}.
    \label{eq:universal-convex-bound}
\end{align}

\end{proof}

\section{Proofs for \pref{sec:applications-games}}
\label{app:applications}
\begin{proof}
Let $\g_t^\x=\nabla_\x f(\x_t,\y_t)+\gamma(\x_t-\x_{t-1})$ and $\g_t^\y=-\nabla_\y f(\x_t,\y_t)+\gamma(\y_t-\y_{t-1})$ be the corrected gradients of the two players, with $\gamma=G/D$, $\x_0=\x_1$, and $\y_0=\y_1$.
Write the empirical gradient variations and cumulative squared movements as
\begin{align*}
    \Vb_T^\x={}&\sumTT\norm{\nabla_\x f(\x_t,\y_t)-\nabla_\x f(\x_{t-1},\y_{t-1})}^2,\\
    \Vt_T^\x={}&\sumTT\norm{\g_t^\x-\g_{t-1}^\x}^2,\qquad S_T^\x=\sumTT\norm{\x_t-\x_{t-1}}^2,\\
    \Vb_T^\y={}&\sumTT\norm{\nabla_\y f(\x_t,\y_t)-\nabla_\y f(\x_{t-1},\y_{t-1})}^2,\\
    \Vt_T^\y={}&\sumTT\norm{\g_t^\y-\g_{t-1}^\y}^2,\qquad S_T^\y=\sumTT\norm{\y_t-\y_{t-1}}^2.
\end{align*}

For the dishonest case, the two-player game degenerates to two separate online convex optimization problems. Therefore, the results for both bilinear and strongly-convex-strongly-concave games follow directly from \pref{thm:universal} since $V_T=\O(T)$.

 For the honest case, we focus on the convex-concave game, where $f(\x, \y)$ is convex in $\x$ and concave in $\y$, which subsumes both bilinear and strongly-convex-strongly-concave games. 
The proof begins by analyzing the regret of player-$\x$, following the structure in \pref{app:universal-proof}.

 For the meta regret, \pref{eq:meta-convex-bound} gives
\begin{equation}
    \label{eq:meta-game}
 \begin{aligned}
\MetaReg^\x\le {}&\sqrt2C_0D\sqrt{\Vt_T^\x}+\sqrt2C_0H\sqrt{S_T^\x}+(C_0+C_1)HD,\\
\MetaReg^\y\le {}&\sqrt2C_0D\sqrt{\Vt_T^\y}+\sqrt2C_0H\sqrt{S_T^\y}+(C_0+C_1)HD.
\end{aligned}
\end{equation}

For the base regret, \pref{eq:universal-convex-base-raw} and \pref{eq:universal-convex-base-variation} give
\begin{align}
    \label{eq:base-game}
    \BaseReg^\x\le \O\sbr{\sqrt{\Vt_T^\x}},\quad
    \BaseReg^\y\le \O\sbr{\sqrt{\Vt_T^\y}}.
\end{align}

Combining \pref{eq:meta-game} and \pref{eq:base-game}, we have
\begin{align*}
&\Reg_T^\x+\Reg_T^\y\\
\le{}& \MetaReg^\x+\MetaReg^\y+\BaseReg^\x+\BaseReg^\y- \frac{\gamma}{2}S_T^\x-\frac{\gamma}{2}S_T^\y+\O(1)\\
\le{}&\O\sbr{\sqrt{\Vt_T^\x}+\sqrt{\Vt_T^\y}+\sqrt{S_T^\x}+\sqrt{S_T^\y}}- \frac{\gamma}{2}S_T^\x-\frac{\gamma}{2}S_T^\y\\
\le{}&\O\sbr{\sqrt{\Vb_T^\x}+\sqrt{\Vb_T^\y}+\sqrt{S_T^\x}+\sqrt{S_T^\y}}- \frac{\gamma}{2}S_T^\x-\frac{\gamma}{2}S_T^\y\le\O(1),
\end{align*}
where the third inequality follows from $\Vt_T^\x\lesssim \Vb_T^\x+S_T^\x$ and $\Vt_T^\y\lesssim \Vb_T^\y+S_T^\y$. The last inequality uses $\Vb_T^\x\lesssim S_T^\x+S_T^\y$ and $\Vb_T^\y\lesssim S_T^\y+S_T^\x$.
\end{proof}

\end{document}